\documentclass[]{IEEEtran}

\DeclareUnicodeCharacter{00A0}{ }
\usepackage[utf8x]{inputenc}
\usepackage{blindtext, graphicx}
\usepackage{amsmath,amsthm,multirow,amssymb}
\usepackage[linesnumbered,ruled,vlined]{algorithm2e}
\usepackage{lscape}
\usepackage{booktabs}
\usepackage{tabularx}
\usepackage{float}

\usepackage{scrextend}
\usepackage[export]{adjustbox}
\usepackage{textcomp}

\usepackage{xcolor}
\usepackage{subcaption}
\newcommand{\comment}[1]{}

\usepackage[sort&compress,comma,square,numbers]{natbib}

\newcounter{remark}
\newenvironment{remark}[1][Remark]{\refstepcounter{remark}\begin{trivlist}
\item[\hskip \labelsep {\bfseries #1 \theremark.}]}{\end{trivlist}}

\newcounter{example}
\newenvironment{example}[1][Example]{\refstepcounter{example}\begin{trivlist}
\item[\hskip \labelsep {\bfseries #1 \theexample.}]}{\end{trivlist}}

\theoremstyle{definition}
\newtheorem{theorem}{Theorem}
\newtheorem{lemma}[theorem]{Lemma}

\def \h {\mathbf h}

\def \w {\mathbf{w}}

\begin{document}

%

\author{Derek~T.~Anderson, Matthew~Deardorff, Timothy~C.~Havens, Siva~K.~Kakula, Timothy~Wilkin, Muhammad~Aminul~Islam, Anthony~J.~Pinar,
Andrew~Buck}

\title{Fuzzy Integral = Contextual Linear Order Statistic
\thanks{D.T.~Anderson, M. Deardorff, and A. Buck are with the Department of Electrical Engineering and Computer Science at University of Missouri, Columbia, MO 65211, USA (e-mail: andersondt@missouri.edu, msdrm8@mail.missouri.edu, buckar@missouri.edu)}
\thanks{M.~Islam is with the Department of Electrical \& Computer Engineering and Computer Science, University of New Haven, Connecticut, CT 06516, USA. (e-mail: amin\_b99@yahoo.com)}
\thanks{A.J.~Pinar is with the Department of Electrical and Computer Engineering, Michigan Technological University, Houghton, MI 49931, USA (e-mail: ajpinar@mtu.edu)}
\thanks{T.C.~Havens and S.K.~Kakula are with the Department of Computer Science, Michigan Technological University, Houghton, MI 49931, USA (e-mail: thavens@mtu.edu, skakula@mtu.edu)}
\thanks{T. Wilkin is with the School of Information Technology, Deakin University, Waurn Ponds, Victoria, 3216, Australia (email: tim.wilkin@deakin.edu.au)}}

\markboth{}%
{Islam \MakeLowercase{\textit{et al.}}:}

\maketitle

\begin{abstract}
The fuzzy integral is a powerful parametric nonlinear function with utility in a wide range of applications, from information fusion to classification, regression, decision making, interpolation, metrics, morphology, and beyond. While the fuzzy integral is in general a nonlinear operator, herein we show that it can be represented by a set of contextual linear order statistics (LOS). These operators can be obtained via sampling the fuzzy measure and clustering is used to produce a partitioning of the underlying space of linear convex sums. Benefits of our approach include scalability, improved integral/measure acquisition, generalizability, and explainable/interpretable models. Our methods are both demonstrated on controlled synthetic experiments, and also analyzed and validated with real-world benchmark data sets.
\end{abstract}


\IEEEpeerreviewmaketitle

\section{Introduction}\label{sec:intro}\vspace{5pt}
\IEEEPARstart{T}{he} \emph{fuzzy integral} (FI) is a parametric nonlinear function with utility in calculus \citep{SugenoRef,SugenoPHD,6918497,SUGENO20131}, inference \citep{grabisch2013fundamentals,Leary}, aggregation \citep{keller1994advances,6722924,MelissaFuzzy,ANDERSON201624}, classification \citep{grabisch1994classification,enhanced_fusion_of_deep}, regression \citep{8858835,eyke,Grabisch2003,ANGILELLA2010277}, decision making \citep{grabisch1992multi,grabisch1996application}, morphology \citep{409969}, interpolation \citep{Grabisch2004TheCI}, metrics \citep{4505367}, and beyond. As the FI has its roots in Calculus it should be no surprise that it can be used to achieve a range of feats in the continuous and discrete universes. The FI makes use of a \emph{fuzzy measure} (FM) \citep{SugenoPHD}; a \emph{type} of capacity. For $N$ inputs, the FM has $2^N$ variables, one for each possible subset of inputs. A number of challenges exist relative to the FI/FM\footnote{We will use the notation FI versus FI/FM hereafter unless there is a specific reason we wish to differentiate between these concepts}. 

\begin{enumerate}
    \item \textbf{(C1) Tractability}: \emph{How do we scale the FI to ``Big N''?} \\ The FM is not naturally scalable, i.e., there are $1,024$ variables for $N=10$ inputs, $2^{N=20}=1,048,576$, and $2^{N=100}$ is nonillion variables. In practice, this impacts factors like data storage, computation, and the compleixty of algorithms that learn the FI from data. 
    \item \textbf{(C2): Acquisition}: \emph{Where do we get the FI from?} \\ Was the FI provided by a human, learned from data (see Section \ref{sec:learn}), knowledge about attributes/criteria \citep{MARICHAL2000641}, differentiation \citep{SUGENO20131}, imputed (e.g., $\lambda$-FM \citep{SugenoPHD}), etc.?
    \item \textbf{C3: Uncertainty}: \emph{Was the FI fully observed?}\\ Did an expert provide the full FM? If the FI was learned, what is the degree of observability for variables \citep{islam2017data}? Furthermore, how is this knowledge used (if at all)?
    \item \textbf{C4: Understandability}: \emph{Why? How? What?} \\ Example questions include, how do we explain a FM and how did the FI combine the data? Methods range from FM indices (e.g., Shapley \citep{shapley_original,grabisch2004axiomatization} and interaction index \citep{murofushi1993techniques}) to data centric indicies \citep{xai_wcci,xai_tetci}, visualization \citep{8015533,Buck}, linguistic summarization \citep{MurrayLing}, trust \citep{xai_wcci,xai_tetci}, operator analysis \citep{PriceOper}, and beyond. 
\end{enumerate}

\begin{figure*}[th]
\centering
\includegraphics[width=0.9\linewidth]{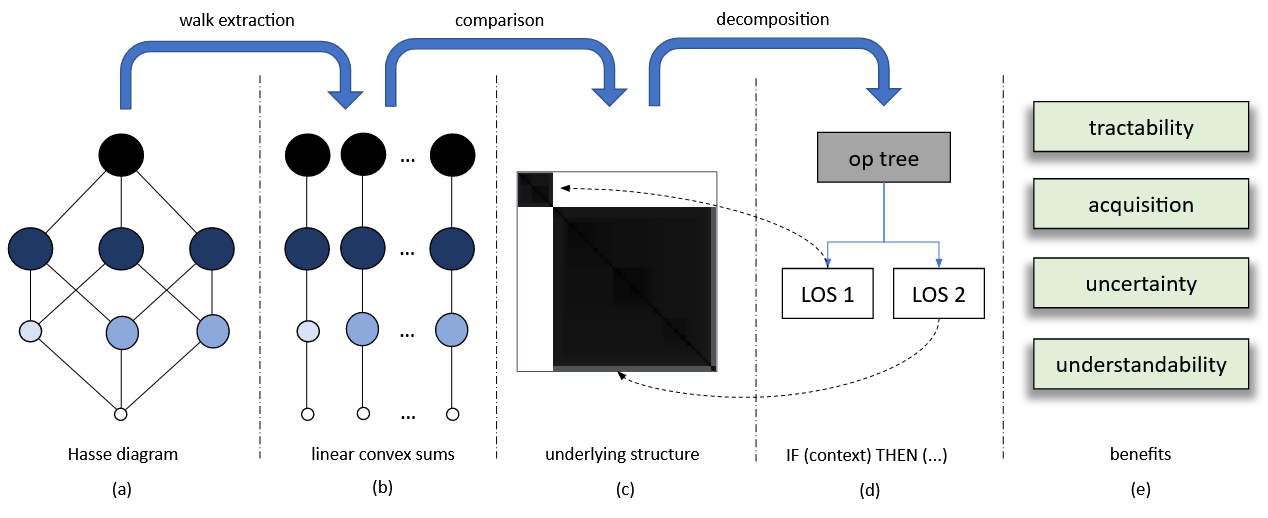}
\caption{Illustration of flow from (a) fuzzy measure to its underlying (b) linear convex sums (LCSs) and (d) set of (approximating) linear order statistics (LOSs). Step (c) is similarity between measures and (d) is clustering for discovery via decomposition.}\label{fig:bigpic}
\end{figure*}

The above list is not comprehensive and it does not focus on topics like integral/measure extensions nor applications. Herein, we show that the FI can be decomposed via a sampling scheme into a set of \emph{linear convex sums} (LCSs), followed by a reduced set of contextual \emph{linear order statistics} (LOS) \citep{yager1988ordered,yager1992applications}. We provide an algorithm to visualize and optionally extract this underlying structure relative to different applications. Figure \ref{fig:bigpic} illustrates the steps in our proposed framework.

As a result of this decomposition, it is possible to seek improved learning algorithms that target the decomposed observable sub-operator structure (aka advancement of C2). We can also obtain a deeper understanding---namely in application but also theoretically---of a particular instance of the FI and its underlying application domain (C4). This capability is a selling point for why use the FI. Our representation also enables mitigation---either decision suppression or imputation (intelligent guessing)---based on how well ``parts'' of a FM have been approximated relative to new instances/samples at \emph{runtime} (C3). Last, this article identifies a tradeoff space in which one can intelligently decide how to best allocate limited resources in the pursuit of scaling the FI to ``bigger $N$'' (C1).

\section{Linear Convex Sum and Linear Order Statistic}\label{sec:los}

In this subsection we quickly review the LCS and LOS. Let $X = \{x_1,...,x_N\}$ be $N$ \emph{sources}, e.g., humans, algorithms, or sensors. Furthermore, let $h(x_i)$ be the ``input'' from source $i$, e.g., subjective belief, probability, objective sensor measurement, etc. For simplicity, let $\h=(h_1,...,h_N)^t$ be a vector of $N$ inputs\footnote{For the remainder of this article, we use the shorthand notation $h_i=h(x_i)$}. The LCS is  

\begin{equation}
f^1_{\w}(\h) = \sum_{i=1}^{N}{ w_i h_i },
\label{eq:lcs}
\end{equation}
\noindent where $\w=(w_1,...,w_N)^t \geq \mathbf{0}^N$ and $\left(\sum_{i=1}^{N}w_i\right)=1$.

The LOS is
\begin{equation}
f^2_{\w}(\h) = \sum_{i=1}^{N}{ w_i h_{\pi{(i)}} },
\label{eq:lcos}
\end{equation}
\noindent where $\pi$ is a sorting such that $h_{\pi(1)} \geq h_{\pi(2)} ... \geq h_{\pi(N)}$. Thus, a LOS is a LCS with a ``pre-sort'' when $\w=(w_1,...,w_N)^t \geq \mathbf{0}^N$ and $\left(\sum_{i=1}^{N}w_i\right)=1$.

\begin{remark}
\textbf{(Familiar Operators)} The LOS generates a wealth of common operators, from t-norms (intersection like) to expected value and t-conorm (union like) operators. For example, $\w=(1,0,...,0)^t$ turns the LOS into the maximum, $\w=(0,...,0,1)^t$ is the minimum, $\w=(\frac{1}{N},...,\frac{1}{N})^t$ is the mean, $\w=(0,...,0,1,0,...0)^t$ is the median, etc. There are numerous ``soft'' and ``trimmed'' versions of these operators. The point is, the LOS produces a wealth of functions outside the \emph{functional scope} of the LCS.
\end{remark}

\section{Fuzzy Measure and Fuzzy Integral}\label{sec:fifm}


The FM, $g:2^X \rightarrow \mathbb{R}^+$, is a function with the following two properties; (i) (boundary condition) $g(\emptyset) = 0$, and (ii) (monotonicity) if $A,B \subseteq X$, and $A \subseteq B$, then $g(A) \le g(B)$\footnote{Sometimes a normality condition is imposed such that $g(X)=1$.}.


The reader can refer to \citep{ANDERSON201624} for a recent survey of different fuzzy integrals from the Sugeno integral \citep{SugenoPHD} to the Sugeno \emph{Choquet integral} (ChI) \citep{MUROFUSHI1991532} and other more recent and exotic extensions (non-direct FI \citep{MelissaFuzzy,ANDERSON201624,6722924}, non-convex and sub-normal FI \citep{ANDERSON201624,6722924}, shape preserving FI \citep{8491555}, k-additive FI \citep{GRABISCH1997167}, symmetric and asymmetric FI \citep{grabsym}, bipolar FIs \citep{GRECO201321}, etc.). The ChI of $\h$ on finite $X$ is 
\begin{align}
\int{\h \circ g} = C_g(\h) = \sum_{j=1}^N h_{\pi(j)} (g(A_{\pi(j)}) - g(A_{\pi(j-1)})),
\end{align}\label{eq:ChI}       
\noindent for $A_{\pi(j)} = \{ x_{\pi(1)},$  $\dots,$ $ x_{\pi(j)}\}$, $g(A_{\pi(0)})=0$, and $\pi$ (sort).

\begin{remark}
\textbf{($\text{LOS} \subset \text{ChI}$)} The LOS is a special case of the ChI; i.e., when $g(A) = g(B)$ for $|A|=|B|$, $A,B \in 2^X$.
\end{remark}

\subsection{Hasse Diagram and Walks}\label{sec:walk}
The FM/capacity and ChI can be visualized as a Hasse diagram; induced by the monotonic inequality constraints. For example, $g(A) \geq \max_{x_k \in A}{g(A \setminus x_k)}, A \subseteq X \setminus \emptyset$. ``Nodes'' (or vertices) in the graph are subsets of $X$ and ``edges'' indicate set exclusion. Relative to a sort of the data ($\pi_i$)---of which there are $N!$ possible sorts---the ChI is
\begin{align}
\sum_{j=1}^N h_{\pi_i(j)} (g(A_{\pi_i(j)}) - g(A_{\pi_i(j-1)})) = \h_{\pi_i}^t\w_{\pi_i},
\end{align}\label{eq:ChIRewrite}  
\noindent where $w_{\pi_i}(j)=(g(A_{\pi_i(j)}) - g(A_{\pi_i(j-1)}))$. It is trivial to prove that $\sum_{j=1}^{N}w_{\pi_i}(j)=1$ when $g(\emptyset)=0$, $g(X)=1$, and $g(A) \geq \max_{x_k \in A}{g(A \setminus x_k)}$. Herein, we refer to $\pi_i$ as a ``walk'' (in the Hasse diagram) since $\w_{\pi_i}$ uses exactly one variable at each ``level'' in the diagram\footnote{In prior work, researchers have refered to these set of weights as a probability distribution given the boundary conditions, positivity, and additive properties. Herein, we highlight this relationship but are careful to not draw any conclusions, e.g., is a walk semantically connected to a random process?}. That is, there is one variable for the empty set, one density, one tuple, etc. Thus, the ChI is simply $N!$ LCSs, one for each possible sort. The key is, the ChI is a reduction from $N! \times N$ LCS variables to $2^N$ FM variables. They are different embeddings. 

\begin{example}\label{ex:chiexp}
For $N=2$, the ChI can be expanded as
\begin{align}
C_g(\h) = \left\{
  \begin{array}{lr}
    h_1 w_1 + h_2 w_2 & : h_1 \geq h_2 \\
    h_2 w_3 + h_1 w_4 & : h_2 > h_1
  \end{array}
\right.
\end{align}
where $w_1 = g(\{x_1\})$, $w_2 = 1 - g(\{x_1\})$, $w_3 = g(\{x_2\})$, $w_4 = 1 - g(\{x_2\})$. Thus, there are four weights but just two underlying \emph{free} FM variables; the densities.
\end{example}

A similar story holds for the LOS, it is a reduction from $N! \times N$ variables to $N$.

\begin{example}\label{ex:lcosexp}
For $N=2$, the LOS can be expanded as
\begin{align}
f^2_{\w}(\h) = \left\{
  \begin{array}{lr}
    h_1 w_1 + h_2 w_2 & : h_1 \geq h_2 \\
    h_2 w_3 + h_1 w_4 & : h_2 > h_1
  \end{array}
\right.
\end{align}
where $w_1 = w_3$ and $w_2 = w_4$. Thus, there are four weights but just two underlying \emph{free} LOS variables.
\end{example}

Last, the LCS and LOS both have $N$ weights, but the LOS sorts the data, inducing the nonlinearity.

\begin{example}\label{ex:lcsexp}
For $N=2$, the LCS can be expanded as
\begin{align}
f^1_{\w}(\h) = \left\{
  \begin{array}{lr}
    h_1 w_1 + h_2 w_2 & : h_1 \geq h_2 \\
    h_2 w_3 + h_1 w_4 & : h_2 > h_1
  \end{array}
\right.
\end{align}
where $w_1 = w_4$ and $w_2 = w_3$. 
\end{example}

In summary, our message is that one can view a class of aggregation operators as mappings from the space of $N! \times N$ weights to a more manageable space, e.g., $2^N$ for the ChI and $N$ for the LCOS and LCS. These observations will help us later with determining how to decompose the FI.

\subsection{Data-Driven Learning Algorithms}\label{sec:learn}

Many methods have been proposed to learn the FI/FM from data. Examples include linear programming \citep{beliakov2009construction}, quadratic programming \citep{grabisch2013fundamentals,8858835,islam2017data,pinar2017measures,andersonbinary,Anderson2014}, evolutionary algorithms \citep{muhammad_choquet_evo}, gradient descent \citep{muhammad_chimp,mendez2008minimum}, Gibbs sampling \citep{mendez2007sparsity}, GOAL programming \citep{7284140}, regularization \citep{pinar2017measures,siva1,Anderson2014}, and reward and punishment \citep{KELLER19961}, multiple instance learning \citep{9002801,7743905}, binary capacities \citep{9002801,andersonbinary}, multi-kernel learning \citep{7337934,7762088}, and more. Last, methods can be divided into those that learn the full measure or a subset (e.g., the densities) coupled with imputation (Sugeno $\lambda$ FM, S-Decomposable FM, etc.). An aim of our current article is to reveal relevant underlying structure to facilitate new scalable data-driven optimizations. 

\subsection{Data Supported FM Variables}\label{sec:support}

In \citep{islam2017data}, we showed that data driven learning of the FI results in missing variables (FM variables not ``seen'' during learning), an interval-valued FM (a direct result of unobserved variables and the FM constraints), and an interval-valued FI result (a result of an interval-valued FM). For example, let $N=3$ and $h_3 > h_1 > h_2$. The FM variables encountered are $g(\{h_3\})$, $g(\{h_1, h_3\})$ and $g(X)$. If we consider all training data then it is trivial to track all data supported variables; those that have at least been observed once. This is a simple binary check for inclusion/exclusion. The reader can imagine extending this to a degree of support (something not yet formally characterized). 

\begin{figure}[tbh]
\centering
\includegraphics[width=0.7\linewidth]{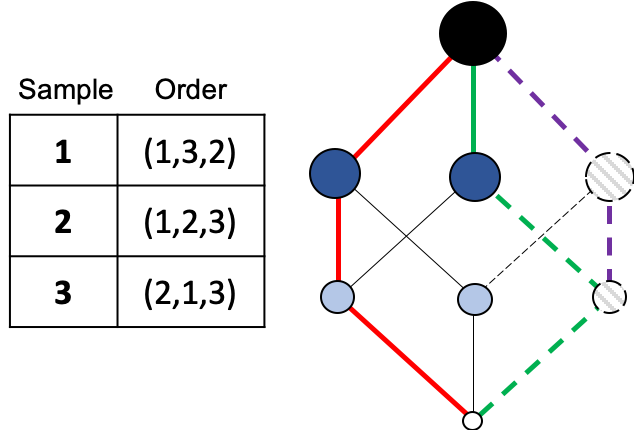}
\caption{Example FM with observed walks reported (table), one observed walk (LCS) highlighted (red), one unobserved walk highlighted (purple), and one unobserved walk shown which has a subset of variables observed and unobserved (green). Dotted lines and dashed variables denote unobserved.}\label{fig:ivatclodd}
\end{figure}

\subsection{Interpretable and/or Explainable FM and FI}\label{sec:xai}

A long standing goal is understanding the measure and integral, both in theory and in practice. Numerous works have appeared, which we organize according to what \emph{parts} they use; (i) the integral, (ii) the FM, and/or (iii) data. Well known examples of (ii) include the Shapley and interaction index; what is the ``worth'' of the individual inputs and how are these inputs ``interacting''. In \citep{xai_tetci}, we extended these indices (which assume total observability) to data centric problems, which are partially observable domains. Another topic is how is the data being combined? In \citep{YAGER1993125}, Yager proposed ORNESS and ANDNESS to measure the degree to which an OWA operator is similar to a disjunction or conjunction. In \citep{price2014indices}, we proposed indices to measure the similarity of a FI to operators like the minimum, maximum, the mean, and an OWA in general. In \citep{Torra2013HellingerDF}, Torra et al. proposed a Hellinger distance-based measure between FMs. Last, there is the question of if an FI was learned from data, what \emph{characteristics} does the data have and how does that impact us? In \citep{xai_wcci}, we proposed data-centric indices that inform what parts (walks and variables) of the FM were learned, their sampling statistics, and a measure of data support-based trust was proposed. Again, this list is not comprehensive but it gives the reader a feel for existing work. 

\section{Decomposition}\label{sec:decomp}

The goal of this section is to decompose the ChI into a set of manageable ``parts'' (LCSs) to improve our understanding, optimization, and discovery of a set of underlying exact or approximate LOS operators. First, we can write the ChI as
\begin{equation}
  C_g(\h) =
  \begin{cases}
    \w_{\pi_1}^t \h & \text{if $\pi^* = \pi_1$} \\
    ... &  \\
    \w_{\pi_k}^t \h & \text{if $\pi^* = \pi_k$} \\
    ... &  \\    
    \w_{\pi_{N!}}^t \h & \text{if $\pi^* = \pi_{N!}$,} \\
  \end{cases}
\end{equation}\label{eq:decomp}
\noindent where $\pi^*$ is the sort for $\h$ (a new instance) and $k$ denotes the $k$th lexicographically encoded sort, e.g., for $N=3$, $\w_{\pi_1}$ is $h_{1} \geq h_{2} \geq h_{3}$, $\w_{\pi_2}$ is $h_{1} \geq h_{3} \geq h_{2}$, etc. $C_g(\h)$ is simply $N!$ LCSs with $2^N$ underlying shared variables (see Lemma \ref{lem:unique}).

\begin{remark}
Equation \ref{eq:decomp} can be interpreted a number of ways. One take is that a (discrete/finite) ChI is simply a collection ($N!$) of context driven logics; where context refers to the input \emph{strength} of our sources and logic is a source specific weighting. For example, ``if source two is more confident than source five which is more confident than ... then take $w_1$ amount of source two's input, $w_2$ of source five, etc.'' 
\end{remark}

\begin{lemma}\label{lem:unique}
A ChI with $2^N$ FM variables has a minimum of one and maximum of $N!$ unique operators.
\end{lemma}
\begin{proof}
This lemma can be solved in two parts. First, let $g(A)=g(B)=g(X)=1$ and $g(\emptyset)=0$. This FM has just one underlying operator, $w(1)=(g(A_{1})-0)=(1-0)=1$, $w(2)=...=w(N)=(1-1)=0$. Second, let each $w_{\pi_k}(j)=(g(A_{\pi_k(j)}) - g(A_{\pi_k(j-1)})) \in [0,1]$ be a unique value, which is possible $\forall$ $N$ as we have a $\Re$-valued domain.
\end{proof}

Lemma \ref{lem:unique} highlights that not every FI/FM has the same number of unique underlying operators. Our goal is to identify the unique ones. Benefits of doing this include:

\begin{enumerate}
    \item \textbf{(B1) Understandability}: identifying underlying unique operators and their contexts (sorts) lets the reader know what logic (operator) is being used and when. 
    \item \textbf{(B2) Storage}: Instead of storing and indexing $2^N$ variables, a problem with fewer unique operators is more efficient, helping us scale to bigger $N$.
    \item \textbf{(B3) Data}: A model is only as good as the data used to learn it. Identifying which contexts/sorts were not encountered helps us understand when a model may fail to generalize. As such, a reader can decide to suppress outputs, identify mitigation or imputation strategies, and it helps the user understand what data needs to be collected to make a model more complete/robust.  
\end{enumerate}

\begin{example}
Consider the following example of fusing a set of decisions generated by different sensors for explosive hazard detection and humanitarian demining. Let source one ($x_1$) be an algorithm processing a \emph{ground penetrating radar} (GPR) sensor, source two ($x_2$) is an algorithm focused on \emph{infrared} (IR) imaging, and source three ($x_3$) is an algorithm processing \emph{electromagnetic induction} (EMI) signals. Hypothetically, let $g(A)=1$ for $\forall A, |A|=2$, $g(X)=1$, and $g(\{x_1\})=0$, $g(\{x_2\})=g(\{x_3\})=1$. This FI/FM has three underlying operators, $\textbf{w}_1=(0,1,0)^t$ (median) if $h_1$ is the largest number (i.e., GPR is the most confident) and $\textbf{w}_2=(1,0,0)^t$ (max) else. This logic can be describe as ``if IR or EMI is the most confident, take that as the answer.'' On the other hand, ``if GPR is the most confident, do not believe it, take the median of the three inputs''. While this is a simple example, it highlights the fact that there are advantages to understanding when (context) to use particular logic (operators). Beyond simply knowing what it is doing, it is possible that on particular domains this knowledge can help us better understand the application itself, e.g., the physics and sensing phenomonology between sensors, environments, and/or objects for explosive hazard detection.
\end{example}


\begin{algorithm}
 \caption{Naive discovery of underlying ChI LOSs} \label{alg:nivalg}
 \KwData{$g$ - Input FM }
 \KwResult{$L$ - Set of underlying LOSs }
 Initialization, $L=\emptyset$ \\
 \For{each Hasse diagram walk, $\pi_k$}{
   \If{($\textbf{w}_{\pi_k} \not\in L$) and ($\Phi(\pi_k) == 0$)}{
     $L = \{ L \cup \textbf{w}_{\pi_k} \}$
    }
 }
 Return $L$.
\end{algorithm}

\begin{algorithm}
 \caption{gSAMP: FM sampling algorithm} 
 \KwData{$g$ - Input FM}
 \KwResult{$X$ - Sample set}
 Operator dataset, $X=\emptyset$ \\
 \For{each Hasse diagram walk, $\pi_k$}{
   \If{$\Phi(\pi_k) < \epsilon$}{
      $X = \{ X \cup \textbf{w}_{\pi_k} \}$ \\
    }
 }
\label{alg:sample}
\end{algorithm}

\begin{figure}[th]
\centering
\includegraphics[width=0.9\linewidth]{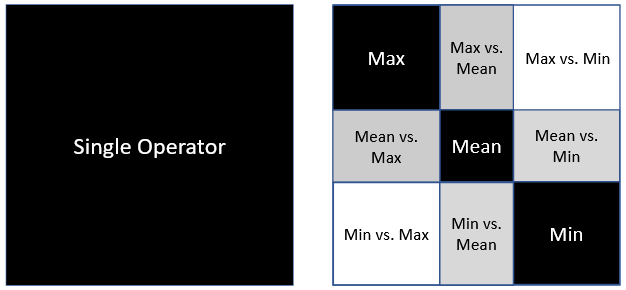}
\caption{iVAT examples. (left) A single operator, e.g., ChI for the max, and (right) a ChI with three underlying LOSs.}\label{fig:vatimg}
\end{figure}

\begin{algorithm}
 \caption{ChI operator visualization} 
 \KwData{$g$ - Input FM}
 $X$ = gSAMP($g$) \\
 $D$ = dissimilarity($X$) \\
 $\hat{D}$ = normalize(iVAT($D$)) \\
 Display $\hat{D}$.
\label{alg:ivat}
\end{algorithm}

\begin{algorithm}
 \caption{ChI operator discovery} 
 \KwData{$g$ - Input FM}
 \KwResult{$L$ - Set of underlying LOSs }
 $X$ = gSAMP($g$) \\
 $D$ = dissimilarity($X$) \\
 $(U,C) = \text{CLODD}(\text{normalize(iVAT}($D$)))$ \\
 \For{each cluster, $c \in C$}{
   Identify cluster medoid, $\textbf{x}_c$ \\
   $L = \{ L \cup \textbf{x}_c \}$ \\
 } 
\label{alg:clodd}
\end{algorithm}

A trivial way to quickly identify all underlying LOSs is Algorithm \ref{alg:nivalg}.\footnote{In Algorithm \ref{alg:nivalg}, $\Phi(\pi_k)$ is how many variables in a walk are unobserved.} However, it falls short due to factors such as:

\begin{enumerate}
    \item \textbf{Sampling Statistics and Noise}: Factors like the number of observed data points relative to the set of all intersecting paths per variable in a walk and/or noise in the data can result in variability in the underlying walks. As a result, more operators could be generated than needed, which is misleading and can lead to over fitting. 
    \item \textbf{Operator Similarity}: Consider two LOSs that vary with respect to a ``tiny'' amount. Algorithm \ref{alg:nivalg} is extreme in the respect that it will identify all unique operators. However, if we were to present these LOSs to a human, it is possible that they consider those operators to be equal. The point is, Algorithm \ref{alg:nivalg} sees the world as black and white and the problem of findings similar underlying operators can be a much more complicated process.
\end{enumerate}

Determining similarity between operators is not trivial and it can depend on factors like sampling, noise, and/or what similarity measure is used. This is the reason why we cast the operator discovery process as a sampling and clustering problem, which is discussed next.


\subsection{LCS Sampling and Clustering for LOS Discovery}\label{sec:walkandclustering}

Next, we outline clustering for aggregation operator discovery. Algorithm \ref{alg:sample} shows how a FM is converted into a dataset. Algorithm \ref{alg:ivat} subjects these samples to cluster tendency analysis to ``look'' at possible underlying aggregation operator structure (see Figure \ref{fig:vatimg}). The \emph{improved visual assessment of cluster tendency} (iVAT) \citep{5710918,Bezdek2002VATAT} is used, and in the case of ``Big N'', BigVAT \citep{HUBAND20051875}. However, if the reader prefers to directly extract clusters, then Algorithm \ref{alg:clodd} uses the \emph{clustering in ordered dissimilarity data} (CLODD) \citep{havensclodd} algorithm. We use iVAT to visualize clustering structure and CLODD for clustering, however other clustering algorithms exist and can be used, e.g., fuzzy c-means, DBSCAN, mean shift, etc. (in light of a globally best clustering algorithm).

Algorithm \ref{alg:clodd} has the following details. \emph{(Observed)}: The function $\Phi(\pi_k)$ returns what percentage of a walk is unobserved; either $\{0,1\}$ if a walk was fully observed or it can be based on what percentage of variables in a walk were observed. \emph{(Medoid)}: Herein, we compute the degree of dissimilarity of each sample in cluster $i$ to all other samples in cluster $i$, we sum these dissimilarities per sample, and the sample with the lowest value is selected; i.e., the sample that is ``most like'' all samples in the cluster. \emph{(Dissimilarity)}: The function dissimilarity$(X)$ calculates pairwise distances between samples in $X$. Examples include the $\ell_p$-norm, A-norm, Torra's Hellinger-based distance \citep{Torra2013HellingerDF}, etc. It should be noted that the weight vector components are not independent due to the summation constraint and hence are not points on a Cartesian product lattice but rather are \emph{compositional data} on the unit simplex (under closure). Computation of the mode as a medoid of such data was presented in \citep{8858850}
 and \citep{WilkinRef}.


\subsection{Dissimilarity Matrix Normalization}\label{sec:norm}
Consider the issue of how to normalize dissimilarity for iVAT visualization. This is typically achieved by scaling the values, based on the maximum observed dissimilarity, to a range of $[0,1]$ so that contrast is maximized and any internal structure is clearly visible. However, we can imagine the case where all walks in the Hasse diagram are approximately the same from a functional perspective. An iVAT image produced with the typical scaling would show complex structure, which could lead to false beliefs about the underlying LOSs present. Instead, we scale the dissimilarity matrix herein based on the maximum possible dissimilarity between two LOSs. For two LOSs scaled between $[0,1]$, the maximum dissimilarity for an $\ell_p$-norm is $2$; as the distance between the max, $a=(1,0,...,0)^t$, and min operators, $b=(0,...,0,1)^t$, is $||a-b||^2_2 = \sum_{i=1}^{N}{(a_i - b_i)^2} = \left( 1^2 + 0 + ... + (-1)^2 \right) = 2$. This results in an iVAT image that more clearly distinguishes between LOSs that are functionally approximate and those that are radically different. 

\section{Experiments and Results}\label{sec:expres}

In this section, we use a combination of controlled synthetic experiments and real-world datasets. The reason for synthetic data is we know the truth. This is critical as we look to understand the impact of \emph{factors}, e.g., number of walks observed, noise, different underlying measures, etc. While this allows us to characterize and understand the proposed methods, the real-world often provides counter cases or factors not imagined. For that reason, we include two applications. The first is the decision level fusion of a set of heterogeneous deep learners for land cover and object classification in remote sensing on community benchmark datasets. The second is regression on the community UCM machine learning datasets. This was selected to provide both a case of classification and regression.


\subsection{Synthetic: Known Operators and Noise}\label{sec:knownnoiseexp}

Figure \ref{fig:noisyfm} shows a handcrafted FM with four underlying LOSs, along with trained versions of that FM with noisy data. The intent of this experiment is to highlight that noise clearly impacts our underlying cluster structure and its subsequent identification; where \emph{noise} is defined as $C_g(\textbf{x})+\sigma$. Figure \ref{fig:noisyfm}(a) is the answer with no noise, and as the noise level increases ($\sigma=0.1$) we can clearly see---Figure \ref{fig:noisyfm}(d)---that a greater number of ``not our true'' operators arise to compensate. The result is clearly a \emph{distorted} model that (over) fits our training data. This example is for illustrative purposes, the reader can refer to integral regularization work such as ours \citep{pinar2017measures} to help combat the impact of noise on model estimation. 

\begin{figure*}
\centering
\subcaptionbox{$\sigma=0$}{\includegraphics[width=.23\linewidth]{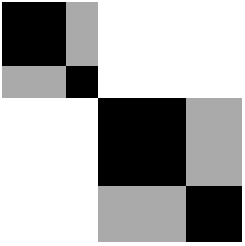}}
\hfill
\subcaptionbox{$\sigma=0.025$}{\includegraphics[width=.23\linewidth]{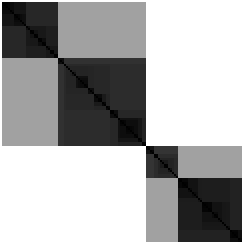}}
\hfill
\subcaptionbox{$\sigma=0.05$}{\includegraphics[width=.23\linewidth]{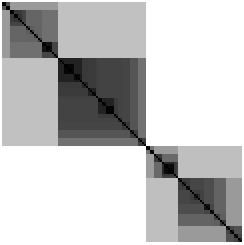}}
\hfill
\subcaptionbox{$\sigma=0.1$}{\includegraphics[width=.23\linewidth]{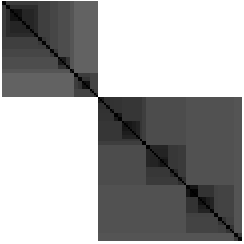}}
\caption{As a progressively noisier Gaussian distribution is added to the samples used to train a FM, the underlying LOSs diverge from the true set of four unique operators. Note that the "order" of the clusters in the iVAT imagery changes between figures. This is because the top-left cluster is always that which is most self similar. The FM used, excluding boundary cases, was $g(A_1)=0.1$ for $|A_1|=1$, $g(A_2)=0.2$ for the first four lexicographically encoded variables relative to $|A_{2}|=2$ and $g(A_3)=0.4$ otherwise for $|A_{3}|=2$, $g(A_4)$ = 0.6 for $|A_4| = 3$, and $g(A_5) = 0.6$ for the first three lexicographically encoded sets relative to $|A_{5}|=4$ and $g(A_6)=0.9$ otherwise for $|A_{6}|=4$.}\label{fig:noisyfm}
\end{figure*}


\subsection{Synthetic: Impact of Sampling on FM Learning}\label{sec:fmlearnexp}

Sampling (volume and variety) drives FM variable observability. Ideally, we would have a sample for each of the $N!$ walks, and multiple observations at that if noise is present. However, this is rarely, if ever, the case in real data. This section highlights (Figure \ref{fig:samplingxxx}) the impact of sampling.

Intuitively, we assumed that operator structure in a full FM iVAT would be poorly represented for the case of low walk visitation. The question being, what about the impact of unobserved variables? However, this is not necessarily as significant and abundant as we anticipated. In many cases the full FM iVAT is almost identical to the observed walk only iVAT image. This is the result of a few factors. For starters, assignment of unobserved FM variable values in learning depends on the underlying solver. For example, consider quadratic programming and iterative solvers such as active set methods. Variables with no data support reside in an \emph{interval of uncertainty} (IoU)---due to the FM boundary and monotonicity constraints---and their values will not be changed from their initial state unless they extend beyond the admissible max (or min, respectively) ranges dictated by the observed variables. 

This brings us to our next point, the specific underlying aggregation operator. For example, the max is one everywhere except the empty set. After a \emph{few} FM variables are observed, the IoUs are empty and the remaining variables are \emph{locked in place}. A similar story holds for the min, locking from above versus locking from below. These are two example extreme LOSs, but similar stories hold arbitrary FMs at that. More often than not, unobserved variables are assigned their IoU extreme bounds, not distorting iVAT/CLODD.  

We recommend that the reader only consider iVAT/CLODD on data supported walks (aka, that what we know). This comes at the expense of simply tracking statistics on observed walks. The reported iVAT plots for full variables versus observed ``look'' similar usually. However, although its not with a high probability, there is too great of a chance of unreliable operator interpretation and identification based on false information factors like random solver initialization.   

\begin{figure*}
\centering
\subcaptionbox{$N=5$, Randomly generated FM, all walks observed\label{fig:randomtrue}}{\includegraphics[width=.23\linewidth]{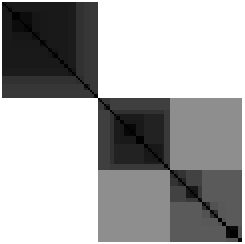}}
\hfill
\subcaptionbox{20\% walk observation}{\includegraphics[width=.23\linewidth]{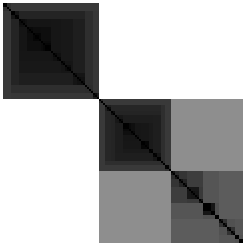}}
\hfill
\subcaptionbox{15\% walk observation }{\includegraphics[width=.23\linewidth]{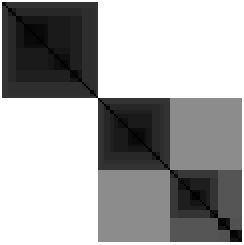}}
\hfill
\subcaptionbox{15\% walk observation, only seen walks}{\includegraphics[width=.23\linewidth]{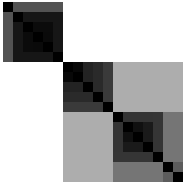}}
\caption{When training data only supports a low percentage of possible walks, we should only decompose and produce iVAT imagery with repect to observed walks, as shown in (d). However, as discussed in Section B, a surprisingly low percentage of walks are usually required to approximate and see similar structure.}\label{fig:samplingxxx}
\end{figure*}

\subsection{Fusion of a Set of Heterogeneous Deep Learners}\label{sec:remotesnesing}

In this section, we use the ChI to fuse a set of heterogeneous architecture \emph{deep convolutional neural networks} (DCNNs) for object detection and land classification in remote sensing. Herein, we fuse seven DCNNs---CaffeNet \citep{caffenet}, GoogleNet \citep{googlenet}, ResNet 50 \citep{resnet}, ResNet 101, DenseNet \citep{densenet}, InceptionResNetV2 \citep{inceptionv4}, and Xception \citep{xception}---on the AID remote sensing dataset \citep{aid}. The AID dataset contains 10,000 images of 30 different aerial scene types. The reader can refer to \citep{CGINets,enhanced_fusion_of_deep} for details about DCNN training. The focus here is not how to learn those parts, but what was learned?

First, AID has $30$ classes. In \citep{CGINets,enhanced_fusion_of_deep}, we trained a ``shared ChI'' (one FM shared across all $30$ classes) and one ChI per class. We performed $5$ fold \emph{cross validation} (CV) with respect to neural learning and $2$ fold CV. Thus, for $5$ fold neural and $2$ fold fusion, and $30$ classes, that is $300$ FMs. Instead of showing $300$ iVAT images, we summarize the trends.\footnote{Note, code is provided at https://github.com/aminb99/choquet-integral-NN and https://github.com/B-Mur/ChoquetIntegral, and the dataset is avilable at https://github.com/aminb99/remote-sensing-nn-datasets, so the reader can reproduce these results if desired.}

Overall, almost all of the learned ChIs were a single operator, the minimum. This makes a lot of sense, as the DCNNs turned out to be a set of strong learners; a scenario not the most ideal for an ensemble. The DCNNs were almost always certain---values near $0$ or $1$---and they were almost always in agreement. It makes sense that the operator learned to take a pessimistic stance. Figure \ref{fig:remotesensing}(b) shows a scaled (between min and max) iVAT image for the case of $N=7$; which has $N!=5,040$ walks. The scaling was necessary because otherwise iVAT is a black image as all LOSs were minimum operators. Occasionally, the method would learn more than one operator, shown in Figure \ref{fig:remotesensing}(a) for the case of $N=4$. The operators learned were $\textbf{w}_1=$ $(0.007,$  $0.491,$  $0.495,$  $0.006)^t$, $\textbf{w}_2=$ $(0.000,$  $0.001,$  $0.994,$  $0.006)^t$, $\textbf{w}_3=$ $(0.000,$  $0.991,$  $0.007,$  $0.001)^t$, which are (approximately) a median, trimmed min, and trimmed max, respectively. Again, this is reassuring, the network did not learn $N!$ different operators but an expected value operator and two trimmed optimistic and pessimistic operators. Last, it is common for the ChIs to observe approximately $60$ of the walks, with almost all observations equating to the sort order $\pi=(1,2,3,...,N)^t$, which is the default walk when all inputs are in agreement.

\begin{figure*}
\centering
\subcaptionbox{$N=4$ (CaffeeNet, GoogleNet, ResNet50, ResNet101), class 1 (agricultural), observed variables only}{\includegraphics[width=.42\linewidth]{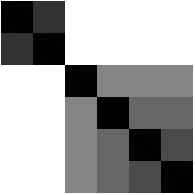}} \hspace{5mm}
\subcaptionbox{$N=7$ (all networks), all FM variables, and iVAT scaled between min and max distances}{\includegraphics[width=.42\linewidth]{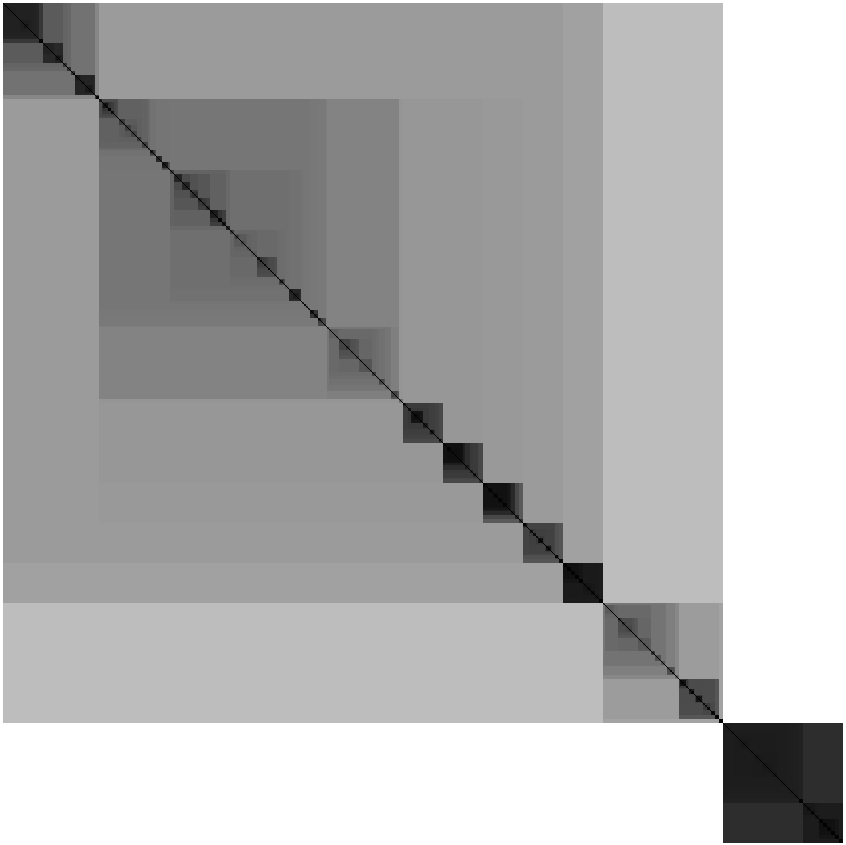}}
\caption{Results on the remote sensing benchmark AID dataset.}\label{fig:remotesensing}
\end{figure*}

\subsection{UCI Machine Learning Repository and Regression}\label{sec:UCI}

In \citep{siva3}, we presented an online learning algorithm for \emph{ChI}-based \emph{regression} (CIR) that scales with $N$. Additionally, $\ell_p$ regularization was used to balance accuracy relative to model complexity. In this section, we use our decomposition procedure to extract explanations from the learned models. This is interesting on two fronts. First, it shows decomposition relative to a different application, regression versus decision level fusion. Second, it shows the impact of regularization on decomposition. The reader can refer to \citep{siva3} for full details on mathematics, learning algorithms, regularization, and statistical performance analysis on community benchmark datasets.


First, our CIR learning algorithm relaxed the montonicity and boundary requirements; which makes it different from the measures shown so far in this article. As a result, the extracted LOSs form tighter clusters. For example, the Yacht dataset without regularization (Figure \ref{fig:yacht}(a)) learned three operators: $\textbf{w}_1 = (-54.6,$ $54.89,$ $-0.24,$ $-0.33,$ $0.54,$ $-0.11)^t,$ $\textbf{w}_2 =$ $(0.05,$ $2.35,$ $-2.35,$ $3.4,$ $-2.35,$ $0.13)^t,$ $\textbf{w}_3 =$ $(-54.6,$ $54.65,$ $-2.59,$ $3.69,$ $-1.12,$ $0.1)^t$. From an efficiency perspective, this tells us that we can compress the $2^N$ variables to $N \times 3$. The Yacht dataset has six features, meaning we only need to keep $18$ versus $64$ weights. Clearly, the savings ratio increases with respect to $N$. For $N=15$ features, three operators would be $45$ versus $32,767$ coefficients; less than one percent ($0.1\%$) of the number of variables. 

In the case of fusion, we discussed LOSs relative to aggregation operators, e.g., max-like, min-like, etc. However, Figure \ref{fig:yacht}(a) is regression. The weights need to be interpreted differently. The bias is what CIR produces for zero-valued inputs. Each weight in $\textbf{w}$ can be thought about in terms of positive and negative correlation. For example, $\textbf{w}_1(1)=-54.6$ is large and negative. Thus, as input one increases, our predicted value decreases (and vice versa). Conversely, $\textbf{w}_1(5)=0.54$ is small(er) and positive input increases are associated with increasing predicted values. On a side note, in \citep{siva3} we normalized each feature to zero mean with unit variance. If this was not the case, then the reader will need to keep in mind feature scaling differences relative to differences in coefficient weights.     

\begin{figure*}
\centering
\subcaptionbox{Full FM, No Reg}{\includegraphics[width=.23\linewidth]{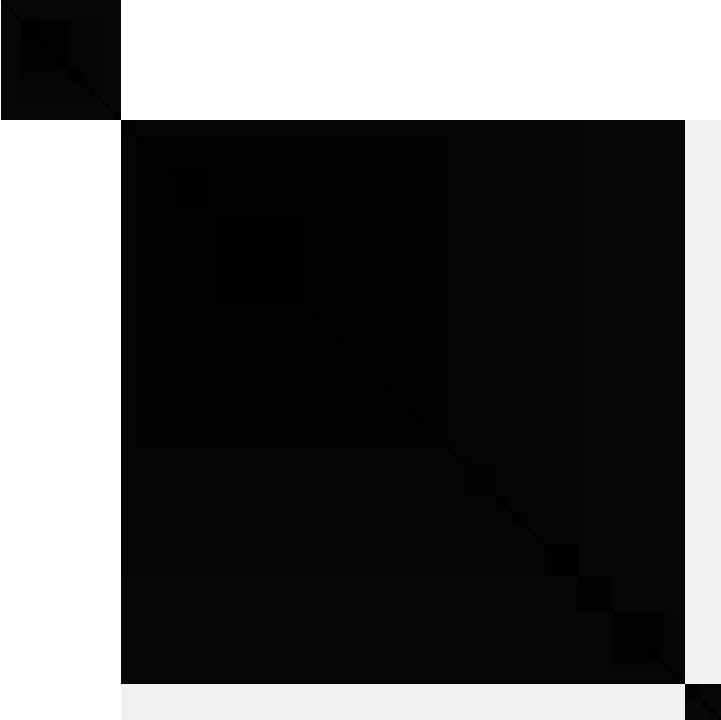}}
\hfill
\subcaptionbox{Full FM, Reg}{\includegraphics[width=.23\linewidth]{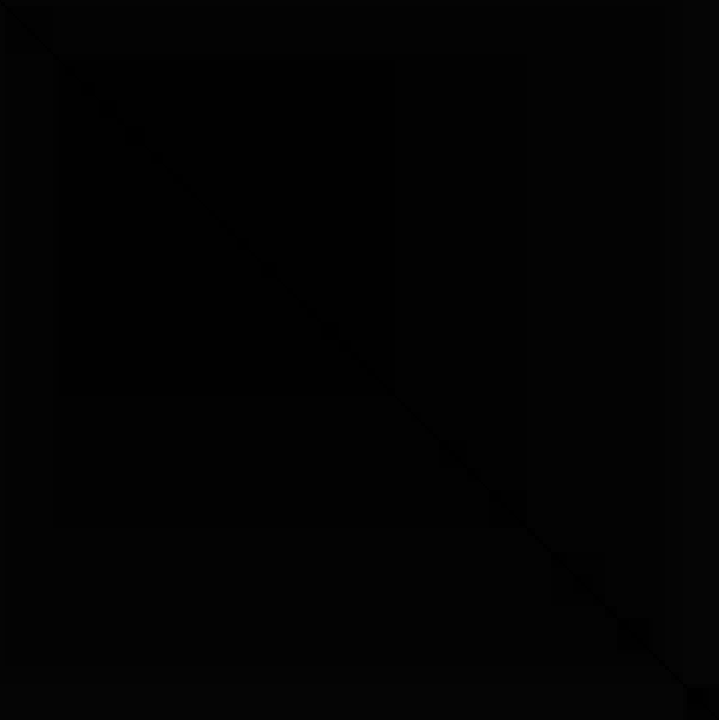}}
\hfill
\subcaptionbox{Walk Subset, No Reg}{\includegraphics[width=.23\linewidth]{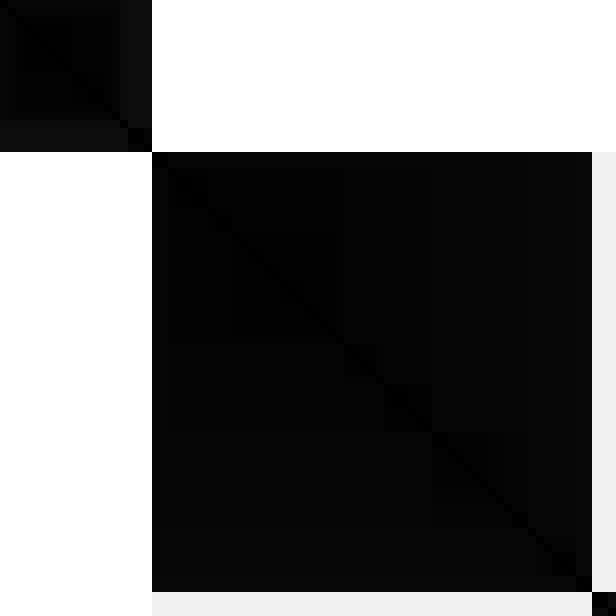}}
\hfill
\subcaptionbox{Walk Subset, Reg}{\includegraphics[width=.23\linewidth]{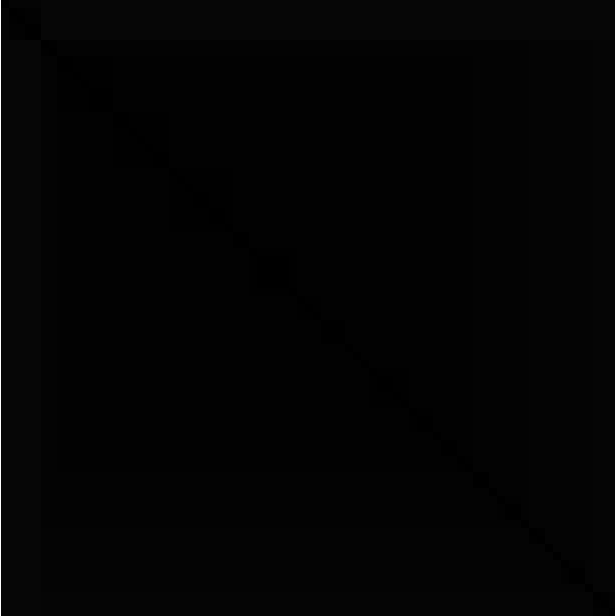}}
\caption{Yacht UCI regression dataset. A total of $10\%$ of the walks were observed. ((a),(b)) and ((c),(d)) are normalized to the same display range for visual display and comparative analysis.}\label{fig:yacht}
\end{figure*}


\begin{figure*}
\centering
\subcaptionbox{Full FM, No Reg}{\includegraphics[width=.23\linewidth]{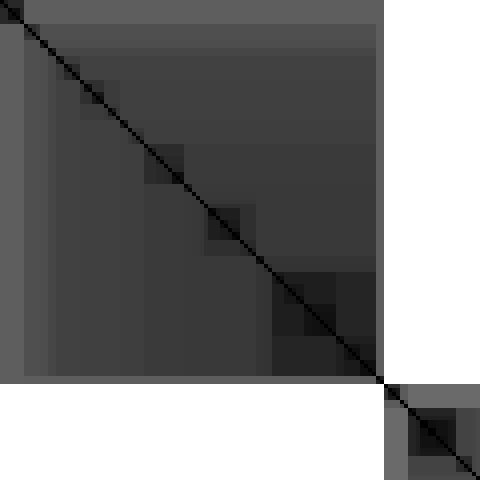}}
\hfill
\subcaptionbox{Full FM, Reg}{\includegraphics[width=.23\linewidth]{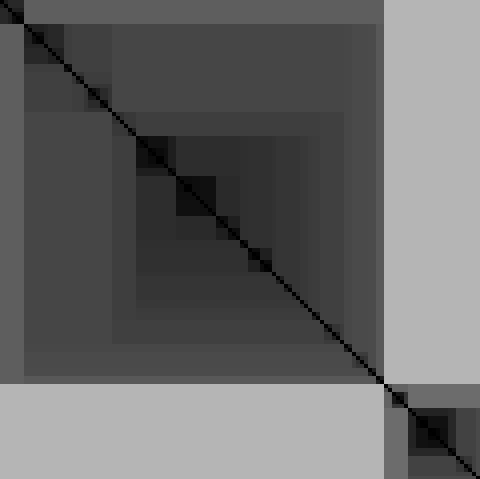}}
\hfill
\subcaptionbox{Walk Subset, No Reg}{\includegraphics[width=.23\linewidth]{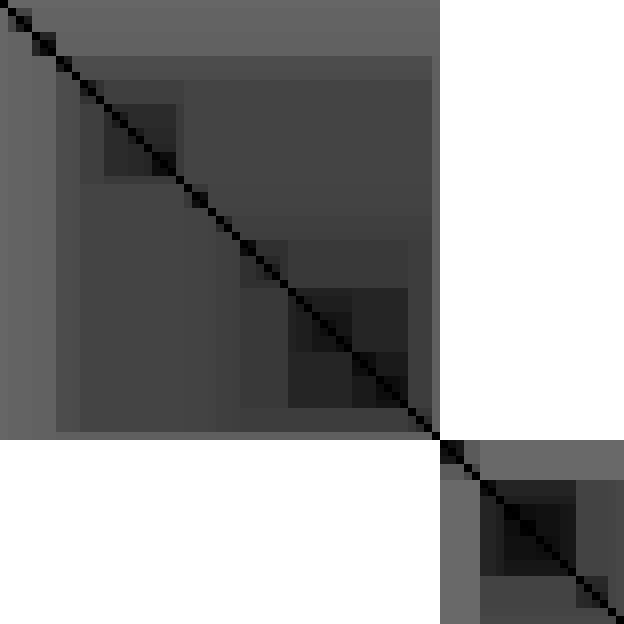}}
\hfill
\subcaptionbox{Walk Subset, Reg}{\includegraphics[width=.23\linewidth]{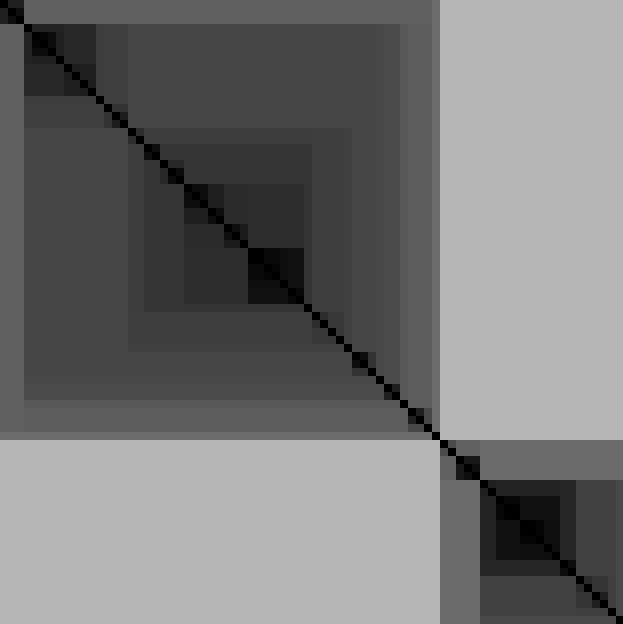}}
\caption{Airfoil UCI regression dataset. A total of only $65\%$ of the walks were observed. ((a),(b)) and ((c),(d)) are scaled to the same minimum and maximum for visual display and comparative analysis.}\label{fig:airfoil}
\end{figure*}


Next, Figure \ref{fig:yacht}(c) shows that our procedure identifies the same cluster structure for all walks, versus just the subset of observed walks. This follows our synthetic experiment discourse in Section \ref{sec:fmlearnexp}. The most interesting---yet predictable---discovery is the impact of regularization. The regularization strategy in \citep{siva3} imposes an $\ell_p$ penalty on the magnitude of FM variables, which promotes sparsity in the FM variables. In many cases, e.g., Figure \ref{fig:airfoil}, the impact of regularization is subtle, if any. However, in other cases, e.g., Figure \ref{fig:yacht}(b) and \ref{fig:yacht}(d), we see that regularization leads to a different and simpler solution; e.g., one underlying equation versus three in the case of Figure \ref{fig:yacht}. We do note that the vast majority of solutions that were learned in our experiments had little-to-no iVAT difference. We cherry picked a few examples to show the reader that it is possible to encounter. Logically, the decomposition's can be different because the cost function has two parts, one component that minimizes functional error relative to training data and a second regularization component that tries to simplify our model to realize a more generalizable solution. Thus, the two error function terms are competitive, not complementary, which can pull their solutions apart.

In summary, the proposed decomposition procedure is valid for regression, the story telling changes with respect to regression coefficient versus decision fusion, and regularization can alter the LOS substructure.\footnote{Note, the way in which regularization impacts LOS substructure is complex and warrants future work. That is, it depends on the \emph{type} of regularization used, see \citep{regintegral}, the underlying ground truth, and the combination thereof.}  

\section{Summary and Future Work}\label{sec:sumfuture}

Herein, we put forth a decomposition procedure for the \emph{fuzzy integral} (FI), specifically the \emph{Choquet integral} (ChI), into a set of underlying \emph{linear order statistics} (LOS). Our procedure involves sampling walks---which are \emph{linear convex sums} (LCSs)---from the \emph{fuzzy measure} (FM), followed by operator similarity and then clustering for LOS discovery. Motivating reasons for seeking a decomposition include tractability (extending the ChI to ``Big N''), acquisition (how do we learn a FI/FM), model uncertainty (and ultimately model generalizability), and understandability (glass versus black box solutions). Our experiments span synthetic to real world and they reveal the behavior of ChI learning relative to noise, random versus structured measures, and sampling statistics (FM variable observation). Real world experiments demonstrated the utility of our methods on decision level fusion for remote sensing and regression on UCI machine learning datasets.

Specifically, we achieved the following with respect to the four challenges (C1-C4) identified in Section \ref{sec:intro}. A direct contribution was made with respect to C4 (understandability), C3 (uncertainty), and C1 (tractability) through the process of contextual operator discovery. However, C2 (acquisition) was not directly advanced herein, e.g., a new learning algorithm to target robust learning of the identified sub-components and/or ways to transfer learned components to unlearned measure structure. In future work, we will strive to combine all four of these challenges into one robust, interpretable, explainable, and actionable solution.    

In future work, we will explore ways to simultaneously learn the decomposition in conjunction with robust statistics of those operators. We will also explore new ways to measure operator similarity, for LOSs and the FM/FI, as existing functions capture syntactic validity but sometimes fall short of capturing semantic similarity. Furthermore, our research highlights the gaps in the FM and FI. New research is needed to intelligently determine how to operate under such conditions, e.g., suppress decision making, interact with a human, transfer learn from past models \citep{8858844} and/or from similar structure in the same model (e.g., imputation), etc. Last, our work identifies observed LCSs and LOSs therefrom. However, even if some walks are not observed, some of their variables have been observed, to some degree by other walks. Last, we showed how to identify fewer LOSs. However, we still have to record which sort orders map to what operators. We need an efficient way to store and index into this reduced operator set.

\bibliographystyle{IEEEtran}
\bibliography{MyBib}

\end{document}